\documentclass{article}

\usepackage[final]{nips_2016}

\usepackage[utf8]{inputenc} 
\usepackage[T1]{fontenc}    
\usepackage{url}            
\usepackage{booktabs}       
\usepackage{amsfonts}       
\usepackage{nicefrac}       
\usepackage{microtype}      

\usepackage{amsmath}
\usepackage{bm}
\usepackage{bbm}
\usepackage{amsfonts}
\usepackage{amssymb,amsthm}
\usepackage{graphicx}

\graphicspath{{../diagrams/}}

\newcommand{\bfx}{\mathbf{x}}

\newcommand{\bfw}{\mathbf{w}}

\newcommand{\bff}{\mathbf{f}}

\newcommand{\Real}[0]{\mathbb{R}}

\theoremstyle{remark}

\providecommand{\theoremname}{Theorem}
\providecommand{\propositionname}{Proposition}
\newtheorem{thm}{\protect\theoremname}
  \theoremstyle{plain}
  \newtheorem{prop}[thm]{\protect\propositionname}

\title{One-vs-Each Approximation to Softmax for Scalable Estimation of Probabilities}

\author{
Michalis K. Titsias\\
Department of Informatics\\
Athens University of Economics and Business\\
\texttt{mtitsias@aueb.gr} \\
}

\begin{document}


\maketitle

\begin{abstract} 
The softmax representation of probabilities for categorical variables 
plays a prominent role in modern machine 
learning with numerous applications in areas such as large scale classification, neural language modeling 
and recommendation systems. However, softmax estimation is very expensive for large scale inference
because of the high cost associated with computing the normalizing constant. 
Here, we introduce an efficient approximation to softmax probabilities 
which takes the form of a rigorous lower bound on the exact probability. This 
bound is expressed as a product over pairwise probabilities 
and it leads to scalable estimation based on stochastic optimization. 
It allows us to perform doubly stochastic estimation 
by subsampling both training instances and class labels. We show that the new 
bound has interesting theoretical properties and we demonstrate its use 
in classification problems. 
\end{abstract}

\section{Introduction}


Based on the softmax representation,  
the probability of a variable $y$ to take the value
$k \in \{1,\ldots,K\}$, where $K$ is the number of categorical symbols or classes,  
is modeled by 
\begin{equation}
p(y=k|\bfx) = \frac{e^{f_k(\bfx; \bfw)}}
{\sum_{m=1}^K e^{f_m(\bfx; \bfw)  }},
\label{eq:softmaxGen}
\end{equation}
where each $f_k(\bfx; \bfw)$ is often referred to as {\em the score function} and it is a real-valued function 
indexed by an input vector $\bfx$ and parameterized by $\bfw$. The score function 
measures the compatibility of input $\bfx$ with symbol $y=k$ so that the higher the score is
the more compatible $\bfx$ becomes with $y=k$. The most common application of softmax is multiclass 
classification where $\bfx$ is an observed input vector and 
$f_k(\bfx; \bfw)$ is often chosen to be a linear function or  more generally a non-linear function such as a neural network \citep{Bishop:2006, Goodfellow-et-al-2016-Book}. 
Several other applications of softmax arise, for instance, in neural language modeling for learning 
word vector embeddings \citep{MnihTeh2012, mikolov2013, pennington-etal-2014}  
and also 
 in collaborating filtering for representing 
probabilities of $(user,item)$ pairs \citep{PaquetKoenigsteinWinther14}. 
In such applications the number of symbols $K$ could often be very large, e.g.\ 
of the order of tens of thousands or millions, which makes the computation of softmax 
probabilities very expensive due to the large sum in the normalizing constant of Eq.\ \eqref{eq:softmaxGen}. 
Thus, exact training procedures based on maximum likelihood 
or Bayesian approaches are computationally prohibitive and approximations are needed. 
%
%
While some  rigorous bound-based approximations to the softmax exists \citep{bouchard_efficient_2007}, they are not so accurate or scalable and therefore it would be highly desirable to develop accurate and computationally efficient 
approximations.

In this paper we introduce a new efficient approximation to softmax probabilities 
which takes the form of a lower bound on the probability of Eq. \eqref{eq:softmaxGen}. 
This bound draws an interesting connection between the exact softmax probability and all its one-vs-each 
pairwise probabilities, and it has several desirable properties. Firstly, for the non-parametric estimation case it leads to an approximation of the likelihood 
that shares the same global optimum with exact maximum likelihood, and thus estimation based on the approximation  
is a perfect surrogate for the initial estimation problem. Secondly, the bound allows for scalable learning 
through stochastic optimization where data subsampling can be combined with 
subsampling categorical symbols. 
Thirdly, whenever the initial exact softmax cost function is convex 
the bound remains also convex. 



Regarding related work, there exist several other methods that try 
to deal with the high cost of softmax such as methods  
that attempt to perform the exact computations \citep{gopal13, VijayanarasimhanSMY14}, methods that change the 
model based on hierarchical or stick-breaking constructions \citep{morin2005hierarchical, KhanMMM12} and 
sampling-based methods \citep{BengioSenecal-2003, mikolov2013, devlin2014, BlackOut}.  
Our method is a lower bound based approach that follows the 
 variational inference framework. Other rigorous variational lower bounds on the softmax 
have been used before \citep{Bohning92, bouchard_efficient_2007}, however they are not easily scalable since they  
require optimizing data-specific variational parameters. In contrast, the bound we 
introduce in this paper does not contain any variational parameter, which 
greatly facilitates stochastic minibatch training. At the same time it 
can be much tighter than previous bounds \citep{bouchard_efficient_2007} as we will 
demonstrate empirically in several classification datasets. 


\section{One-vs-each lower bound on the softmax \label{sec:theory}} 

Here, we derive the new bound on the softmax (Section \ref{sec:onevsone})  
and we prove its optimality property when performing approximate maximum likelihood 
estimation (Section \ref{sec:optimality}). Such a property holds for 
the {\em non-parametric case}, where we estimate probabilities of the 
form $p(y=k)$, without conditioning on some $\bfx$, so that the score 
functions $f_k(\bfx; \bfw)$ 
reduce to unrestricted parameters $f_k$; see Eq.\ \eqref{eq:softmax1} below.
Finally, we also analyze the related bound derived by 
Bouchard  \citep{bouchard_efficient_2007} and we compare it with our approach 
(Section \ref{sec:bouchnonnegsample}).

\subsection{Derivation of the bound \label{sec:onevsone}}

Consider a discrete random variable $y \in \{1,\ldots,K\}$ that takes the value 
$k$ with probability, 
\begin{equation}
p(y=k) = \text{Softmax}_k(f_1,\ldots,f_K) = \frac{e^{f_k}}
{\sum_{m=1}^K e^{f_m}},
\label{eq:softmax1}
\end{equation}
where each $f_k$ is a free real-valued scalar parameter. We wish to express a lower bound 
on $p(y=k)$ and the key step of our derivation is to re-write $p(y = k)$ as
\begin{equation}
p(y=k) = \frac{1}
{1 + \sum_{m \neq k} e^{- (f_k - f_m)}}.
\label{eq:softmax2}
\end{equation}
Then, by exploiting the fact that for any non-negative numbers $\alpha_1$ and $\alpha_2$ it holds 
$1 + \alpha_1 + \alpha_2 \leq 1 + \alpha_1 + \alpha_2 + \alpha_1 \alpha_2 = (1 + \alpha_1) (1 + \alpha_2)$, and more generally it holds 
$(1 + \sum_{i} \alpha_i) \leq \prod_{i} (1  +  \alpha_i)$ where each $\alpha_i \geq 0 $, we 
obtain the following lower bound on the above probability,
\begin{equation}
p(y=k) 
\geq \prod_{m \neq k} \frac{1}
{1 +  e^{- (f_k - f_m)}}
=   \prod_{m \neq k} \frac{e^{f_k}}
{ e^{f_k} +  e^{f_m}} =  \prod_{m \neq k} \sigma(f_k - f_m) . 
\label{eq:softmaxbound}
\end{equation}
where $\sigma(\cdot)$ denotes the sigmoid function. 
Clearly, the terms in the product are pairwise probabilities each corresponding
to the event $y=k$ conditional on the union of pairs of events, i.e.\  $ y \in \{k,m \}$ where 
$m$ is one of the remaining values. We will refer to this bound as one-vs-each bound 
on the softmax probability, since it involves $K-1$ comparisons of a specific event $y=k$ 
versus each of the $K-1$ remaining events. Furthermore, the above result can be stated 
more generally to define  bounds on arbitrary probabilities as 
the following statement shows.  

\begin{prop} Assume a probability model with state space $\Omega$ and probability measure $P(\cdot)$.  
For any event $A \subset \Omega$ and an associated countable set of disjoint events $\{B_i\}$ 
such that $ \cup_{i} B_i = \Omega \setminus A$, 
it holds
\begin{equation}
P(A) \geq \prod_{i} P(A|A \cup B_i).
\label{eq:generalbound}
\end{equation}
\end{prop}
\begin{proof} Given that $P(A) = \frac{P(A)}{P(\Omega)} = \frac{P(A)}{P(A) + \sum_i P(B_i)}$, the result follows by applying the inequality $(1 + \sum_{i} \alpha_i) \leq \prod_{i} (1  +  \alpha_i)$ exactly as done above for the softmax parameterization.
\end{proof}

{\bf Remark.} If the set  $\{B_i\}$ consists of a single event $B$ then by definition $B = \Omega \setminus A$ and the bound is exact since in such case $P(A|A \cup B) = P(A)$. 

Furthermore, based on the above construction we can express 
a full class of hierarchically ordered bounds. For instance, 
if we merge two events $B_i$ and $B_j$ into a single one, then 
the term  $P(A|A \cup B_i) P(A|A \cup B_j)$ in the initial bound is 
replaced with $P(A|A \cup B_i \cup B_j )$ and the associated  
new bound, obtained after this merge, can only become tighter. 
To see a more specific example in the softmax probabilistic model,
assume a small subset of categorical symbols $\mathcal{C}_k$, that does not include $k$, 
and denote the remaining symbols excluding $k$ as $\mathcal{\bar{C}}_k$ so that 
$k \cup \mathcal{C}_k \cup \mathcal{\bar{C}}_k = \{1, \ldots, K\}$. Then, a 
tighter bound, that exists higher in the hierarchy, 
than the one-vs-each bound (see Eq.\ \ref{eq:softmaxbound}) 
takes the form,
\begin{equation}
p(y=k) \geq \text{Softmax}_k(f_k, \bff_{\mathcal{C}_k}) \times  \text{Softmax}_k(f_k,\bff_{\mathcal{\bar{C}}_k}) 
\geq \text{Softmax}_k(f_k, \bff_{\mathcal{C}_k}) \times \prod_{m \in \mathcal{\bar{C}}_k} \sigma(f_k - f_m),
\end{equation} 
where $\text{Softmax}_k(f_k, \bff_{\mathcal{C}_k}) = \frac{e^{f_k}}{e^{f_k} +  \sum_{m \in \mathcal{C}_k} e^{f_m} }$
and $\text{Softmax}_k(f_k, \bff_{\mathcal{\bar{C}}_k}) = \frac{e^{f_k}}{e^{f_k} +  \sum_{m \in \mathcal{\bar{C}}_k} e^{f_m} }$.
 For simplicity of our presentation in the remaining of the paper we do not discuss further 
these more general bounds and we focus only on the one-vs-each bound.  
  
The computationally useful aspect of the bound in Eq.\ (\ref{eq:softmaxbound}) is 
that it factorizes into a product, where each factor depends only on a pair of parameters 
$(f_k,f_m)$. Crucially, this avoids the evaluation of the normalizing constant associated 
with the global probability in Eq.\ \eqref{eq:softmax1} 
and, as discussed in Section \ref{sec:classification}, it leads to  
scalable training using stochastic optimization that can deal with very large $K$. 
Furthermore, approximate maximum likelihood estimation based on the bound 
can be very accurate and, as shown in the next section, it is exact for the non-parametric 
estimation case.  
 
The fact that the one-vs-each bound in \eqref{eq:softmaxbound} 
is a product of pairwise probabilities suggests that there is a connection
with Bradley-Terry (BT) models \citep{bradley1952rank, Huang:2006} for learning individual skills from paired comparisons
and the associated multiclass classification systems obtained by combining binary classifiers, such as one-vs-rest and 
one-vs-one approaches \citep{Huang:2006}. 
Our method differs from BT models, since we do not combine binary probabilistic models 
to a posteriori form a multiclass model. Instead, we wish to develop scalable approximate algorithms that 
can surrogate the training of multiclass softmax-based models by maximizing lower bounds on the 
exact likelihoods of these models.  

\subsection{Optimality of the bound for maximum likelihood estimation
\label{sec:optimality}}
 
Assume a set of observation $(y_1,\ldots,y_N)$ where each 
$y_i \in \{1,\ldots, K\}$.  The log likelihood of the data takes  the form, 
\begin{equation}
\mathcal{L}(\bff) = \log \prod_{i=1}^N p(y_i) = \log \prod_{k=1}^K p(y=k)^{N_k},
\label{eq:exactlik}
\end{equation}
where $\bff = (f_1,\ldots,f_K)$ and $N_k$ denotes the number of data points with value $k$.  
By substituting $p(y=k)$ from Eq.\ (\ref{eq:softmax1}) and then taking derivatives with respect to $\bff$
 we arrive at the standard stationary conditions of the maximum likelihood solution,
\begin{equation}
 \frac{ e^{f_k}}
{\sum_{m=1}^K e^{f_m} } = \frac{N_k}{N}, \ k=1,\ldots,K.
\label{eq:analyticPk}
\end{equation}
These stationary conditions are satisfied for $f_k = \log N_k  + c$ where $c \in \Real$ 
is an arbitrary constant. 
What is rather surprising is that 
the same solutions $f_k = \log N_k  + c$ satisfy also the stationary conditions when maximizing 
a lower bound on the exact log likelihood obtained from  the product of 
one-vs-each probabilities.

More precisely, by replacing $p(y=k)$ with the bound from Eq.\ (\ref{eq:softmaxbound}) we obtain 
a lower bound on the exact log likelihood,
\begin{equation}
\mathcal{F}(\bff) = \log \prod_{k=1}^K \left[ \prod_{m \neq k}  
\frac{e^{f_k}}
{e^{f_k} + e^{f_m}}   \right]^{N_k} 
= \sum_{k > m} \log P(f_k,f_m), 
\label{eq:lowerboundlik}
\end{equation}
where $P(f_k,f_m) = \left[
\frac{e^{f_k}}
{e^{f_k} + e^{f_m}}   \right]^{N_k}  \left[
\frac{e^{f_m}}
{e^{f_k} + e^{f_m}}   \right]^{N_m}$ is a  
likelihood involving only the data of the pair of states 
$(k,m)$, while there exist $K (K-1)/2$ possible such pairs. 
If instead of maximizing the exact log likelihood from Eq.\ \eqref{eq:exactlik}
we maximize the lower bound we obtain the same parameter estimates. 

\begin{prop} The maximum likelihood parameter estimates $f_k = \log N_k  + c, k=1,\ldots,K$ for the exact log likelihood from Eq.\ (\ref{eq:exactlik}) globally also
maximize the lower bound from Eq.\ (\ref{eq:lowerboundlik}). 
\end{prop}
\begin{proof} By computing the derivatives of $\mathcal{F}(\bff)$ 
we obtain the following stationary conditions 
\begin{equation}
K - 1 = \sum_{m \neq k} \frac{N_k + N_m}{N_k}  \frac{e^{f_k}}{e^{f_k} + e^{f_m}}, \ k=1,\ldots,K, 
\end{equation}
which form a system of $K$ non-linear equations over the unknowns $(f_1,\ldots,f_K)$.
By substituting the values $f_k = \log N_k  + c$ we can observe that all $K$ equations are simultaneously satisfied which means that these values are solutions. 
Furthermore, since $\mathcal{F}(\bff)$ is a concave function of $\bff$ 
we can conclude that the solutions $f_k = \log N_k  + c$ globally maximize $\mathcal{F}(\bff)$.
\end{proof}
{\bf Remark.} Not only is $\mathcal{F}(\bff)$ globally maximized by setting 
$f_k = \log N_k  + c$, but also each pairwise likelihood $P(f_k,f_m)$ in Eq.\ 
(\ref{eq:lowerboundlik}) is separately maximized by the same setting 
of parameters. 

\subsection{Comparison with Bouchard's bound \label{sec:bouchnonnegsample}}
 
Bouchard \citep{bouchard_efficient_2007} proposed a related bound 
that next we analyze in terms of its ability to approximate the exact maximum likelihood training 
in the non-parametric case, and then we 
compare it against our method. Bouchard \citep{bouchard_efficient_2007} was motivated 
by the problem of applying variational Bayesian inference to multiclass classification 
and he derived the following upper bound on the log-sum-exp function, 
\begin{equation}
\log \sum_{m=1}^K e^{f_m} \leq \alpha + \sum_{m=1}^K \log \left(1 + e^{f_m - \alpha} \right),
\label{eq:bouchard}
\end{equation}
where $\alpha \in \Real$ is a variational parameter that needs to be optimized in order for the 
bound to become as tight as possible. 
The above induces a lower bound on the softmax probability 
$p(y=k)$ from Eq.\ \eqref{eq:softmax1} that takes the form 
\begin{equation}
p(y=k) \geq \frac{e^{f_k - \alpha}}{\prod_{m=1}^K \left( 1 + e^{f_m - \alpha} \right)}.
\label{eq:softmaxBou}
\end{equation}
This is not the same as Eq.\ (\ref{eq:softmaxbound}), 
since there is not a value for $\alpha$ 
for which the above bound will reduce to our proposed one. For instance, 
if we set $\alpha = f_k$, then Bouchard's bound becomes half the one in Eq.\ 
(\ref{eq:softmaxbound}) due to the extra term $1 + e^{f_k - f_k} = 2$ in the product 
in the denominator.\footnote{Notice that the 
product in Eq.\ (\ref{eq:softmaxbound}) excludes the value $k$, while Bouchard's bound includes it.} 
Furthermore, such a value for 
$\alpha$ may not be the optimal one and in practice $\alpha$ must be
chosen by minimizing the upper bound in Eq.\ \eqref{eq:bouchard}. 
While such an optimization is a convex problem, it requires iterative optimization 
since there is not in general an analytical solution for $\alpha$.  
However, for the simple case where $K=2$ we can analytically find the 
optimal $\alpha$ and the optimal $\bff$ parameters. 
The following proposition carries out this analysis and provides a clear understanding  
of how Bouchard's bound behaves when applied for approximate maximum likelihood estimation.   

\begin{prop} Assume that $K=2$ and we approximate the probabilities
$p(y=1)$ and $p(y=2)$ from (\ref{eq:softmax1}) with the corresponding 
Bouchard's bounds given by $\frac{e^{f_1 - \alpha}}{(1 + e^{f_1 - \alpha}) 
(1 + e^{f_2 - \alpha})}$  and $\frac{e^{f_2 - \alpha}}{(1 + e^{f_1 - \alpha}) 
(1 + e^{f_2 - \alpha})}$. These bounds are used to approximate the maximum
likelihood solution by maximizing a bound $\mathcal{F}(f_1,f_2,\alpha)$ 
which is globally maximized for  
\begin{equation}
\alpha = \frac{f_1 + f_2}{2}, \ \ 
f_k = 2 \log N_k  +  c, \ \ k=1,2.  
\label{eq:Bouchalphaf1f2}
\end{equation}
\end{prop}
The proof of the above is given in the Appendix. 
Notice that the above estimates are biased so that the probability of the most populated 
class (say the $y=1$ for 
which $N_1>N_2$) is overestimated while the other probability is underestimated. This 
is due to the factor $2$ that multiplies $\log N_1$ and  $\log N_2$ in 
\eqref{eq:Bouchalphaf1f2}. 

Also notice that the solution $\alpha = \frac{f_1 + f_2}{2}$ is not a general trend, i.e.\
for $K>2$ the optimal $\alpha$ is not the mean of $f_k$s. In such cases approximate maximum 
likelihood estimation based on Bouchard's bound requires iterative optimization. 
Figure \ref{fig:toycomparisons}a shows some estimated softmax probabilities, using a dataset of 
$200$ points each taking one out of ten values, where $\bff$ is found by exact maximum likelihood, 
the proposed one-vs-each bound and Bouchard's method. As expected estimation based on the bound in Eq.\ \eqref{eq:softmaxbound} gives the exact probabilities, while Bouchard's bound tends to overestimate 
large probabilities and underestimate small ones. 

\begin{figure*}[!htb]
\centering
\begin{tabular}{ccc}
{\includegraphics[scale=0.23]
{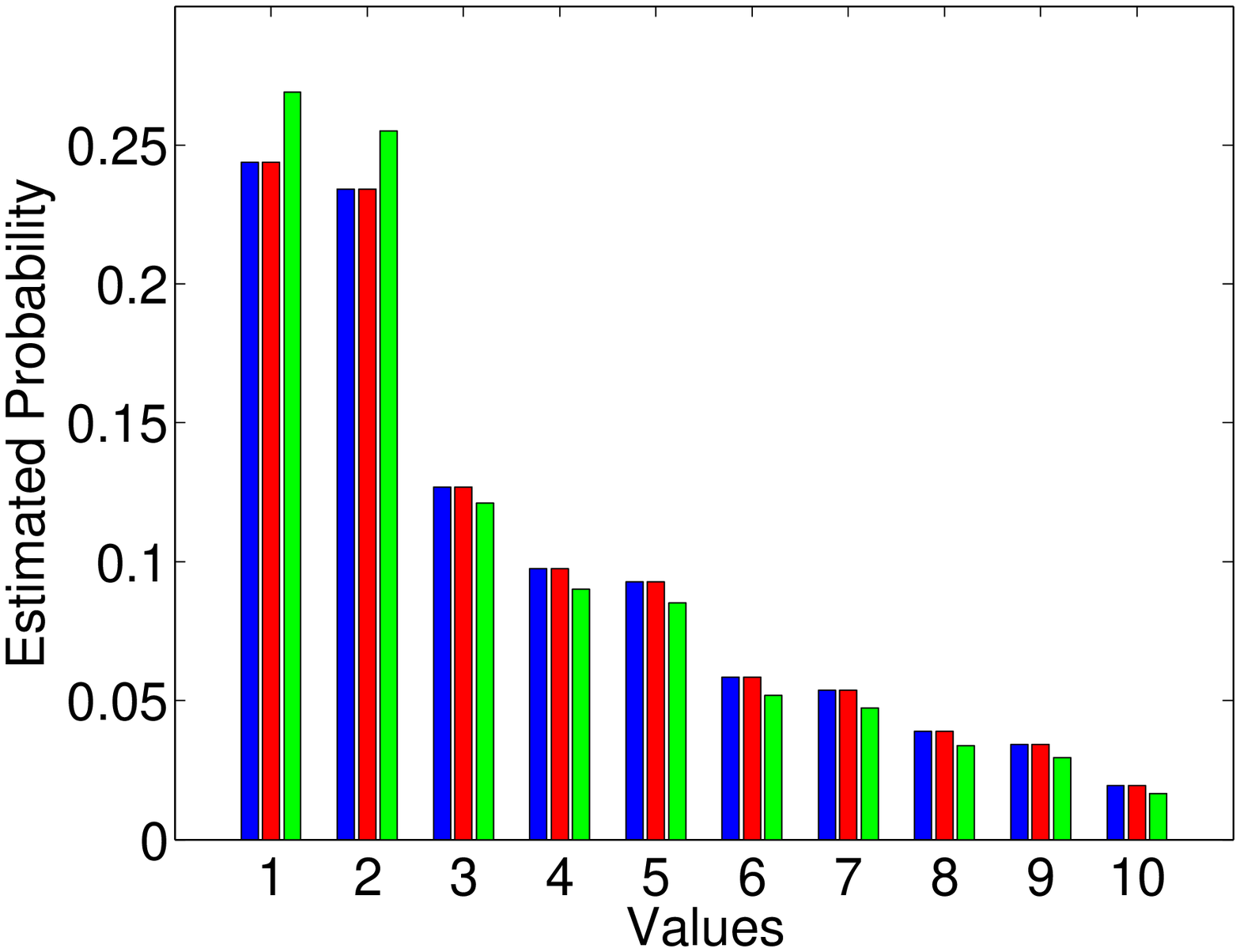}} &
{\includegraphics[scale=0.23]
{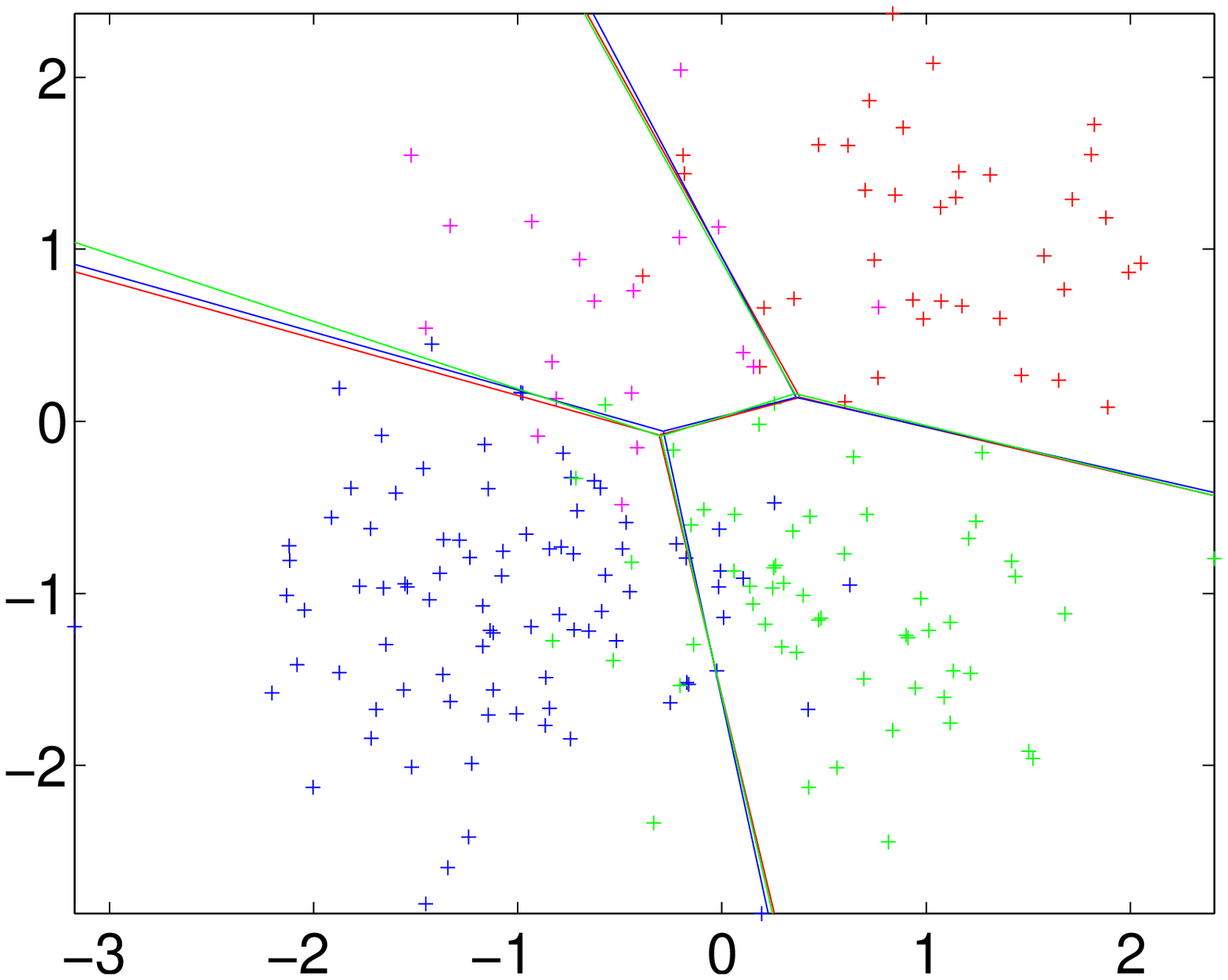}} &
{\includegraphics[scale=0.23]  
{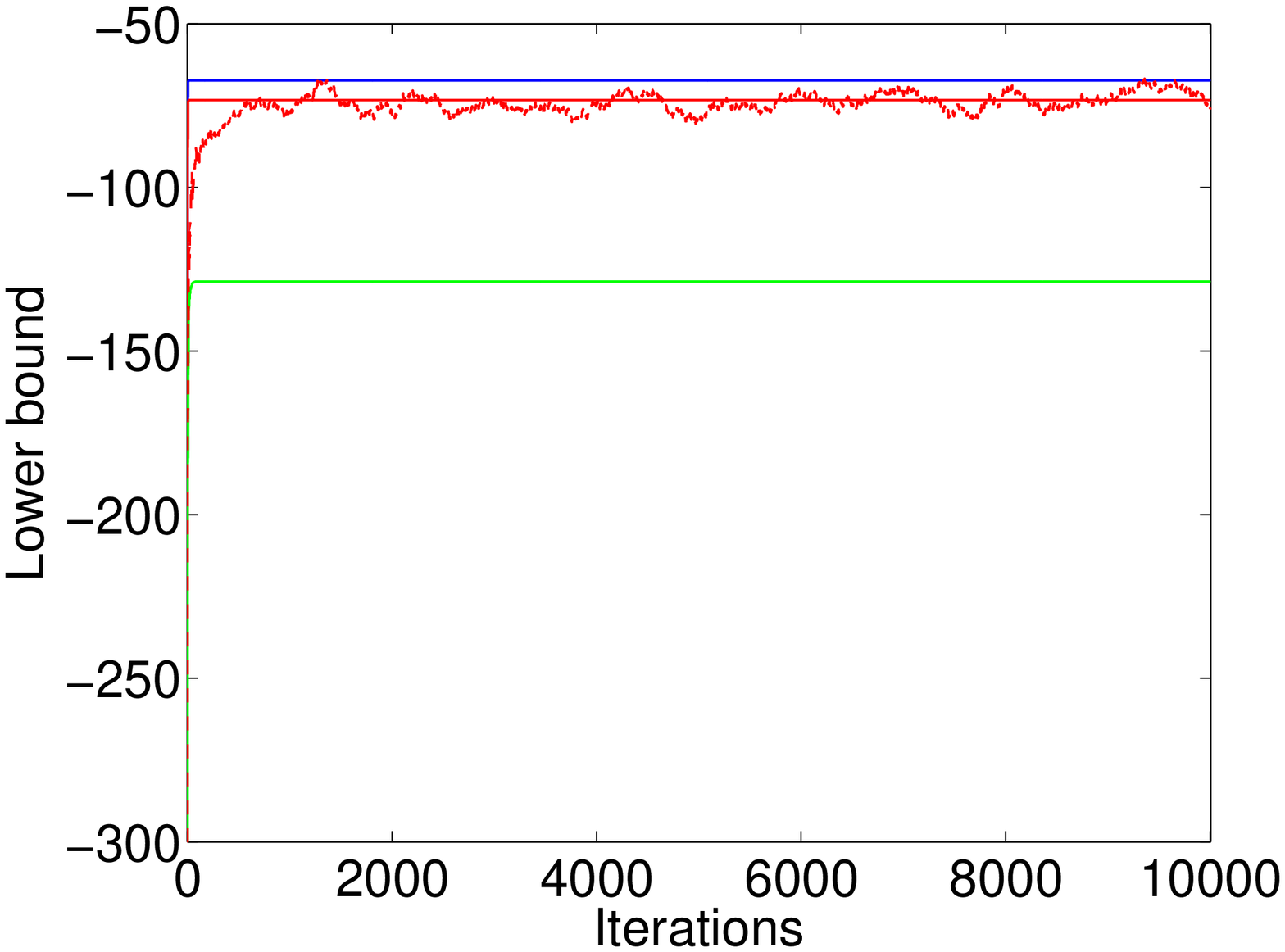}} \\
(a) & (b)  & (c) 
\end{tabular}
\caption{(a) shows the probabilities estimated by exact softmax (blue bar), one-vs-each approximation (red bar) and Bouchard's 
method (green bar). (b) shows the 5-class artificial data together with the decision boundaries found by exact softmax (blue line), one-vs-each (red line) and Bouchard's bound (green line). 
(c) shows the maximized (approximate) log likelihoods for the different approaches when applied to 
the data of panel (b) (see Section \ref{sec:classification}). Notice that the blue line in (c) is the exact maximized log likelihood while the remaining lines correspond to lower bounds. 
 } 
\label{fig:toycomparisons}
\end{figure*}



 \vspace{-4mm}
 
\section{Stochastic optimization for extreme 
classification \label{sec:classification}}

\vspace{-2mm}

Here, we return to the general form of the softmax probabilities as defined by 
Eq.\ (\ref{eq:softmaxGen}) where the score functions 
are indexed by input $\bfx$ and parameterized by $\bfw$. We consider 
a classification task where given a training set $\{\bfx_n, y_n \}_{n=1}^N$, where $y_n \in \{1,\dots,K\}$, we wish to fit the parameters $\bfw$ by maximizing 
the log likelihood, 
\begin{equation}
\mathcal{L} =  \log \prod_{n=1}^N \frac{e^{f_{y_n}(\bfx_n; \bfw)}} {\sum_{m=1}^K e^{f_m(\bfx_n; \bfw)  }}.
\label{eq:Lwx}
\end{equation}
When the number of training instances is very large, the above maximization can be carried out by applying stochastic gradient descent (by minimizing $-\mathcal{L}$) 
where we cycle over minibatches. 
However, this stochastic optimization procedure cannot deal with large values of $K$
because the normalizing constant in the softmax couples all scores functions so that the log likelihood 
cannot be expressed as a sum across class labels. 
To overcome this, we can use the one-vs-each lower bound on the softmax probability from Eq.\ (\ref{eq:softmaxbound})
and obtain the following lower bound on the previous log likelihood, 
\begin{equation}
\mathcal{F} = \log \prod_{n=1}^N \prod_{m \neq y_n} 
\frac{1} {1 + e^{ - [f_{y_n}(\bfx_n; \bfw)  - f_m(\bfx_n; \bfw) ] }} 
= - \sum_{n=1}^N \sum_{m \neq y_n} 
\log \left(1 + e^{ - [f_{y_n}(\bfx_n; \bfw)  - f_m(\bfx_n; \bfw) ]} \right)
\label{eq:onevsonecostClass}
\end{equation}
which now consists of a sum over both data points and labels. Interestingly, 
the sum over the labels, $\sum_{m \neq y_n}$, runs over all remaining 
classes that are different from the label $y_n$ assigned to $\bfx_n$.
Each term in the sum is a logistic regression cost, that depends on the pairwise score difference $f_{y_n}(\bfx_n; \bfw) - f_m(\bfx_n; \bfw)$,
and encourages the $n$-th data point to get separated from the $m$-th remaining class. The above lower bound can be optimized by  stochastic gradient descent 
by subsampling terms in the double sum in Eq.\ (\ref{eq:onevsonecostClass}), thus resulting in a doubly stochastic approximation scheme. Next we further discuss
the stochasticity associated with subsampling remaining classes. 

The gradient for the cost associated with a single training instance $(\bfx_n, y_n)$ is  
\begin{equation}
\nabla \mathcal{F}_n = \sum_{m \neq y_n} \sigma\left( f_m(\bfx_n; \bfw) - f_{y_n}(\bfx_n; \bfw) \right) \left[ \nabla_{\bfw} f_{y_n}(\bfx_n; \bfw)  -   \nabla_{\bfw} f_m(\bfx_n; \bfw) \right]. 
\label{eq:prodgrad}
\end{equation}
This gradient consists of a weighted sum  where the sigmoidal weights   
$\sigma\left( f_m(\bfx_n; \bfw) - f_{y_n}(\bfx_n; \bfw) \right)$ quantify the contribution of the remaining classes to the whole gradient;
the more a remaining class overlaps with $y_n$ (given $\bfx_n$) the higher its contribution is.
A simple way to get an unbiased stochastic estimate of \eqref{eq:prodgrad} is to randomly subsample 
a small subset of remaining classes from the set $\{m | m \neq y_n\}$. 
More advanced schemes could be based on
importance sampling where we introduce a proposal distribution $p_{n}(m)$ defined on the set $\{m | m \neq y_n\}$ that could favor 
selecting classes with large sigmoidal weights. While such more advanced schemes could reduce variance, 
they require prior knowledge (or on-the-fly learning) about how classes overlap with one another. 
Thus, in Section \ref{sec:experiments} we shall experiment only with the simple random subsampling approach and leave the above advanced schemes
for future work.  

To illustrate the above stochastic gradient descent algorithm we simulated a two-dimensional data set of $200$ instances, shown 
in Figure \ref{fig:toycomparisons}b, that belong to five classes. We consider a linear classification model where the score
functions take the form $f_k(\bfx_n, \bfw) = \bfw_k^T \bfx_n$ and where the full set of parameters is $\bfw = (\bfw_1,\dots,\bfw_K)$. 
We consider minibatches of size ten to approximate 
the sum $\sum_n$ and subsets of remaining classes of size one to approximate 
$\sum_{m \neq y_n}$. Figure  \ref{fig:toycomparisons}c
shows the stochastic evolution of the approximate log likelihood (dashed red line), i.e.\ the unbiased subsampling based approximation of \eqref{eq:onevsonecostClass},
together with the maximized exact softmax log likelihood (blue line), the non-stochastically maximized approximate lower bound from   
\eqref{eq:onevsonecostClass} (red solid line) and Bouchard's method (green line). To apply Bouchard's method we 
construct a lower bound on the log likelihood by replacing each softmax probability with the bound from  
\eqref{eq:softmaxBou} where we also need to optimize a separate variational parameter $\alpha_n$ for each data point. 
As shown in Figure \ref{fig:toycomparisons}c our method provides a tighter lower bound than Bouchard's method despite the 
fact that it does not contain any 
variational parameters. Also, Bouchard's method can become very slow when combined with stochastic gradient descent since it requires tuning a separate variational parameter $\alpha_n$ for each 
training instance. Figure \ref{fig:toycomparisons}b also shows the decision boundaries discovered by the exact 
softmax, one-vs-each bound and Bouchard's bound. 
Finally, the actual parameters values 
found by maximizing the one-vs-each bound were remarkably close (although not identical) 
to the parameters found by the exact softmax.  


\vspace{-2mm}

\section{Experiments \label{sec:experiments}} 

\vspace{-2mm}

\subsection{Toy example in large scale non-parametric estimation \label{sec:largedensity}}
\vspace{-1mm}

Here, we illustrate the ability to stochastically maximize the bound 
in Eq.\ \eqref{eq:lowerboundlik} for the simple nonparametric estimation case. 
In such case, we can also maximize the bound based on the analytic formulas 
and therefore we will be able to test how well the 
stochastic algorithm can approximate the optimal/known solution. We consider 
a data set of $N = 10^6$ instances 
each taking one out of $K = 10^4$ possible categorical values. The data were generated 
from a distribution $p(k) \propto u_k^2$, where each $u_k$ was randomly chosen in 
$[0,1]$. The probabilities estimated based on the analytic formulas are shown 
in Figure \ref{fig:DensityLarge}a. To stochastically estimate these probabilities  we 
follow the doubly stochastic framework of Section \ref{sec:classification} so that 
we subsample data instances of minibatch size $b=100$ and for each instance 
we subsample $10$ remaining categorical values. 
We use a learning rate initialized to $0.5/b$ (and then decrease it by a factor of $0.9$ after each 
epoch) and performed $2 \times 10^5$ iterations. 
Figure \ref{fig:DensityLarge}b shows the final values for the estimated probabilities, 
while Figure \ref{fig:DensityLarge}c shows the evolution of the estimation error 
during the optimization iterations. We can observe that the algorithm performs 
well and exhibits a typical stochastic approximation convergence.  

\begin{figure*}[!htb]
\centering
\begin{tabular}{ccc}
{\includegraphics[scale=0.22]
{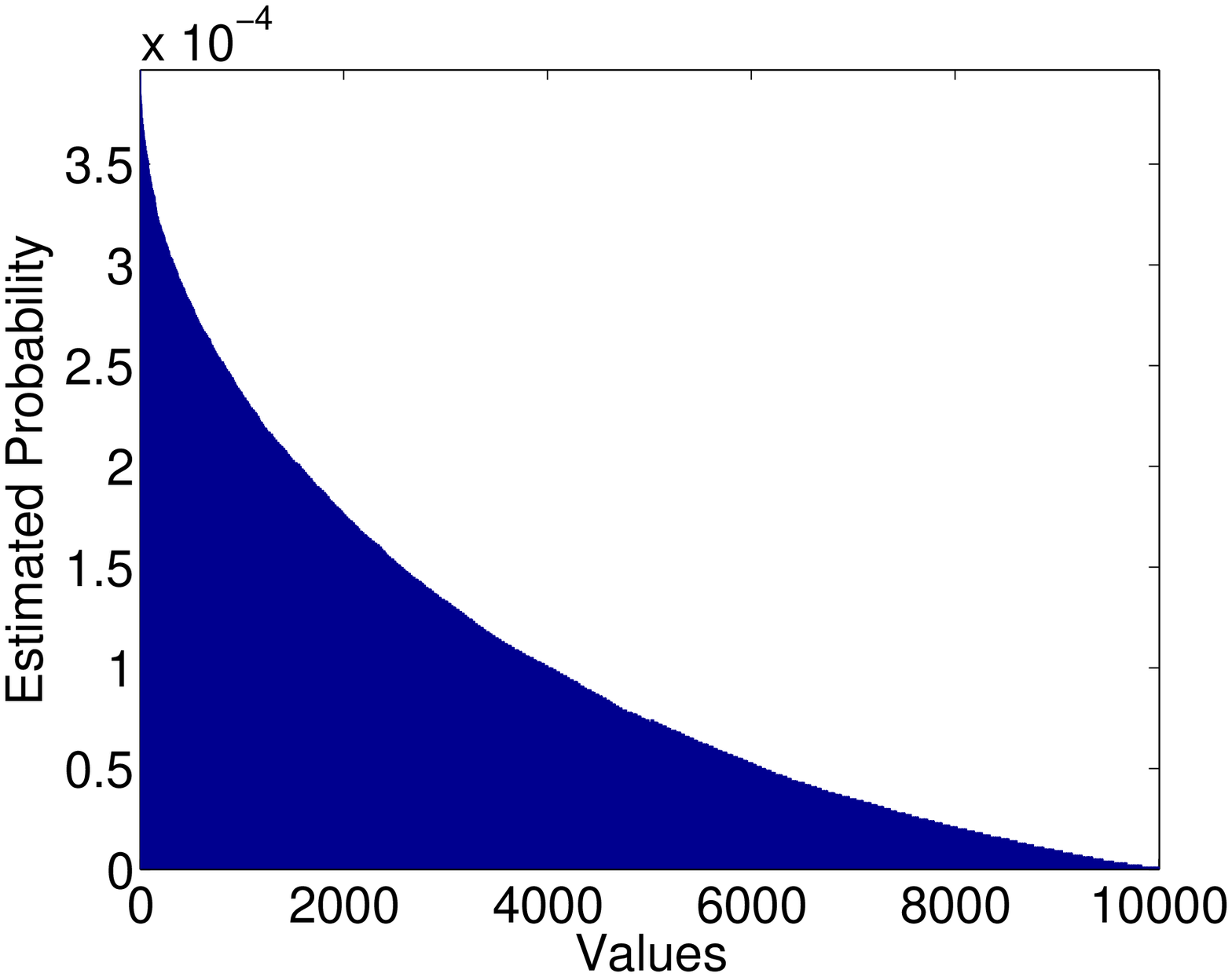}} &
{\includegraphics[scale=0.22]
{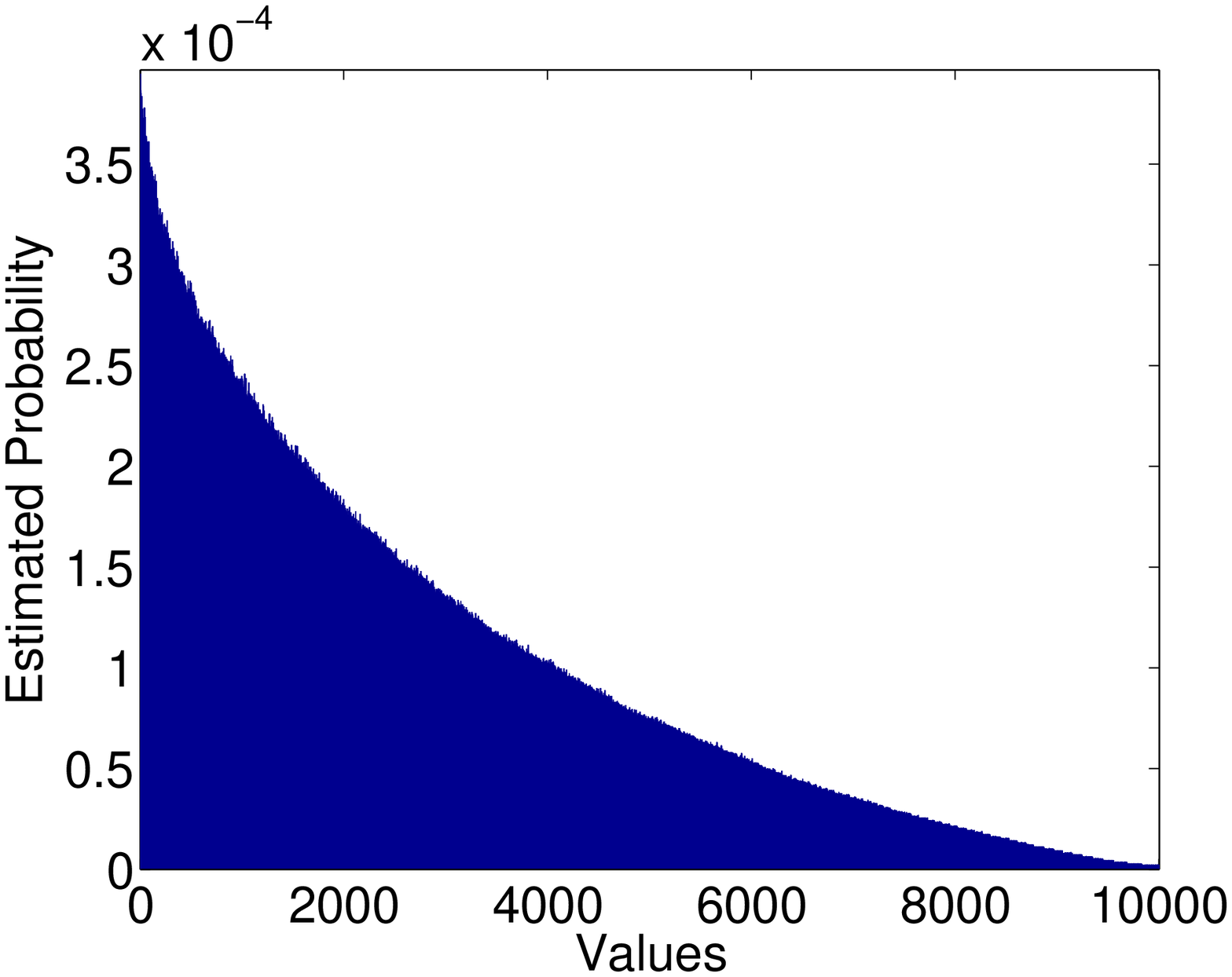}} &  
{\includegraphics[scale=0.22]  
{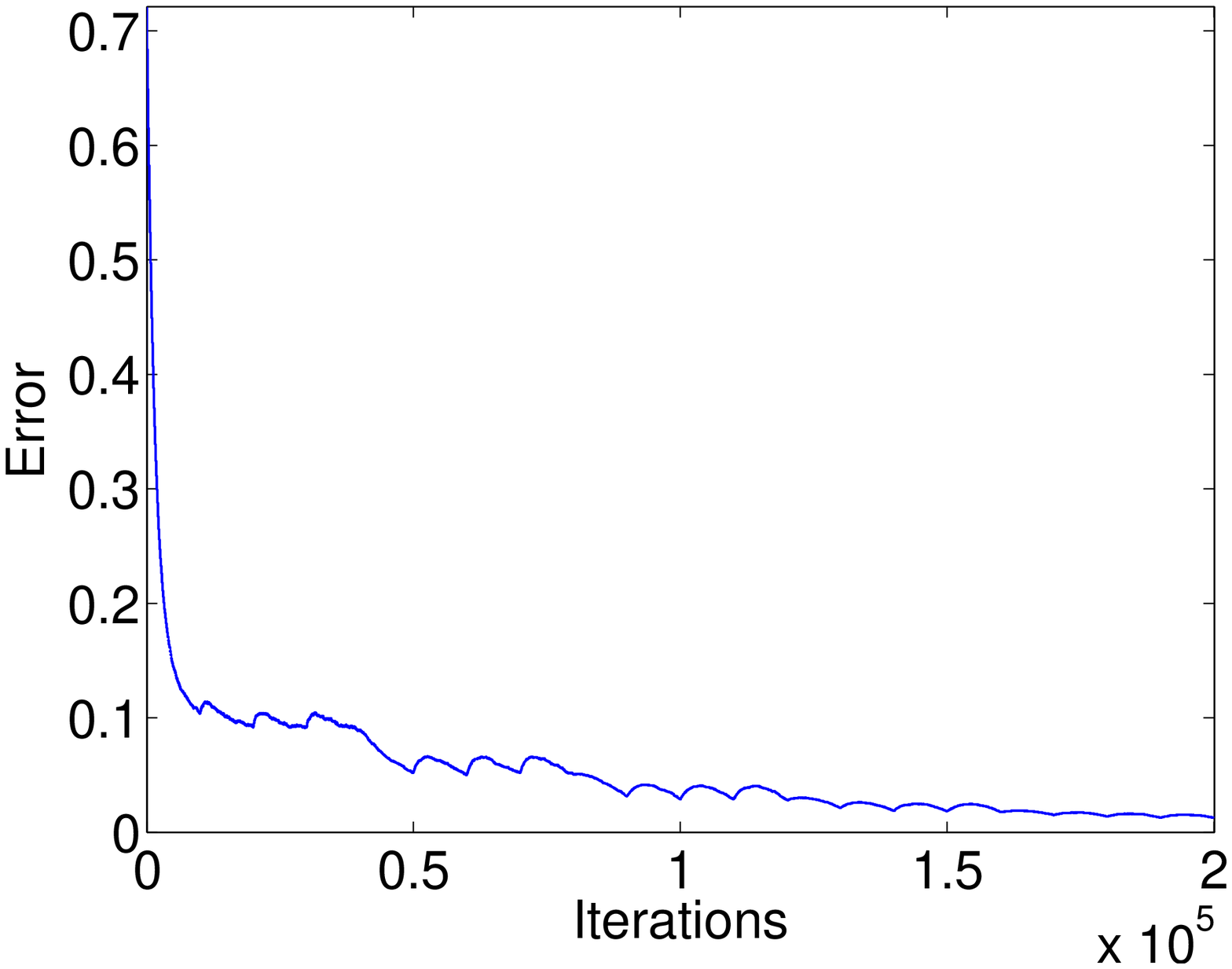}} \\
(a) & (b) & (c)
\end{tabular}
\vspace{-2mm}
\caption{(a) shows the optimally estimated probabilities which have been sorted for visualizations purposes. 
(b) shows the corresponding probabilities estimated by stochastic optimization. (c) shows the 
absolute norm for the vector of differences between exact estimates and stochastic estimates.} 
\label{fig:DensityLarge}
\end{figure*}


\vspace{-3mm}
\subsection{Classification \label{sec:experimClass}} 
\vspace{-2mm}

{\bf Small scale classification comparisons.} 
Here, we wish to investigate
whether the proposed lower bound on the softmax is a good surrogate for 
exact softmax training in classification. More precisely, we wish to 
compare the  parameter estimates obtained by the one-vs-each 
bound with the estimates obtained by exact softmax training. 
To quantify closeness we use the normalized absolute 
norm   
\begin{equation}
\text{norm} = \frac{|\bfw_{\text{softmax}}  - \bfw_*|}{|\bfw_{\text{softmax}}|},
\label{eq:norm}
\end{equation}
where $\bfw_{\text{softmax}}$ denotes the parameters obtained by exact softmax 
training and $\bfw_*$ denotes estimates obtained by approximate training. 
Further, we will also report predictive performance measured by classification 
error and negative log predictive density (nlpd) averaged across test data,
\begin{equation}
\text{error} = (1/N_{test}) \sum_{i=1}^{N_{test}} I( y_i \neq t_i), \quad 
\text{nlpd} =  (1/N_{test}) \sum_{i=1}^{N_{test}} - \log p(t_i|\bfx_i), 
\end{equation}  
where $t_i$ denotes the true label of a test point and $y_i$ the predicted one. 
We trained the linear multiclass model of Section \ref{sec:classification} 
with the following alternative methods: exact softmax training (\textsc{soft}), 
the one-vs-each bound (\textsc{ove}), 
the stochastically optimized  one-vs-each bound (\textsc{ove-sgd}) and
 Bouchard's bound (\textsc{bouchard}). For all approaches, the associated cost function 
was maximized  together with an added regularization penalty term, 
$-\frac{1}{2} \lambda ||\bfw||^2$, which ensures that the global maximum of the cost 
function is achieved for finite $\bfw$.
Since we want to investigate how well we surrogate exact 
softmax training, we used the same fixed value $\lambda=1$ in all experiments.    
  
We considered three small scale multiclass classification datasets:   
\textsc{mnist}\footnote{\url{http://yann.lecun.com/exdb/mnist}},  
\textsc{20news}\footnote{\url{http://qwone.com/~jason/20Newsgroups/}} and \textsc{bibtex} 
\citep{Katakis08multilabeltext}; see Table \ref{table:datasets} for details. 
 Notice that \textsc{bibtex} is originally a
 multi-label classification dataset \citep{NIPS2015_5969}. 
where each example may have more than one labels. 
Here, we maintained only a single label for each data point in order to apply
standard multiclass classification. The maintained label was the first label 
appearing in each data entry in the repository files\footnote{\url{http://research.microsoft.com/en-us/um/people/manik/downloads/XC/XMLRepository.html}} from which we obtained the data. 

Figure \ref{fig:smallScaleClass} displays convergence 
of the lower bounds (and for the exact softmax cost) for all methods. 
Recall, that the methods \textsc{soft}, \textsc{ove} and \textsc{bouchard} 
are non-stochastic and therefore their optimization can be carried out  
by standard gradient descent. Notice that in all three datasets 
the one-vs-each bound gets much closer to the exact softmax cost
compared to Bouchard's bound. Thus, \textsc{ove} tends to give a tighter bound 
despite that it does not contain any variational parameters, while 
\textsc{bouchard} has $N$ extra variational parameters, i.e.\ 
as many as the training instances.  
The application of \textsc{ove-sgd} method (the stochastic version of \textsc{ove}) 
is based on a doubly stochastic scheme where we subsample minibatches of size $200$ and subsample 
remaining classes of size one. We can observe that \textsc{ove-sgd} is able to stochastically approach its 
maximum value which corresponds to \textsc{ove}. 

Table \ref{table:scores} shows the parameter closeness score 
from Eq.\ \eqref{eq:norm} as well as the classification predictive scores. 
We can observe that \textsc{ove} and \textsc{ove-sgd} 
provide parameters closer to those of \textsc{soft} than the parameters provided by 
\textsc{bouchard}. Also, the predictive scores for \textsc{ove}  and  \textsc{ove-sgd} 
are similar to \textsc{soft}, although they tend to be slightly worse. Interestingly, 
\textsc{bouchard} gives the best classification error, even better than the exact softmax training, 
but at the same time it always gives the worst nlpd which suggests sensitivity to overfitting. 
However, recall that the regularization parameter $\lambda$ was fixed to the value one and it was not optimized
separately for each method using cross validation.    
Also notice that \textsc{bouchard} cannot be easily scaled up (with stochastic optimization) 
to massive datasets since it introduces an extra variational parameter for each training 
instance.  

{\bf Large scale classification.} Here, 
we consider \textsc{amazoncat-13k} (see footnote 4) which is a  
large scale classification dataset. 
This dataset is originally multi-labelled \citep{NIPS2015_5969} and here 
we maintained only a single label, as done for the \textsc{bibtex} dataset,
in order to apply standard multiclass classification. This dataset is also highly 
imbalanced since there are about $15$ classes having the half of the training instances while 
they are many classes having very few (or just a single) training instances. 
   
\begin{table}[t]
  \caption{Summaries of the classification datasets.}
  \label{table:datasets}
  \centering
  \begin{tabular}{lllll}
    \toprule
    Name       & Dimensionality  & Classes & Training examples & Test examples \\
    \midrule
    \textsc{mnist}      & 784        & 10      & 60000  & 10000  \\
    \textsc{20news}     & 61188      & 20      & 11269  & 7505 \\
    \textsc{bibtex}     & 1836       & 148     & 4880    & 2515  \\
   \textsc{amazoncat-13k}   & 203882     & 2919    & 1186239  &  306759 \\   
    \bottomrule
  \end{tabular}
\end{table}

\begin{table}[t]
  \caption{Score measures for the small scale classification datasets.}
  \label{table:scores}
  \centering
  \begin{tabular}{lllll}
    \toprule
    &  \textsc{soft}  &  \textsc{bouchard} & \textsc{ove}  & \textsc{ove-sgd}  \\ 
 &  (error, nlpd)  &  (norm, error, nlpd)  &  (norm, error, nlpd)  & (norm, error, nlpd) \\ 
\midrule
\textsc{mnist}  &  (0.074, 0.271)  & (0.64, 0.073, 0.333)  & (0.50, 0.082, 0.287)  &  (0.53, 0.080, 0.278)  \\ 
\textsc{20news}  &  (0.272, 1.263)  & (0.65, 0.249, 1.337)  & (0.05, 0.276, 1.297)  &  (0.14, 0.276, 1.312)  \\ 
\textsc{bibtex}   &  (0.622, 2.793)  & (0.25, 0.621, 2.955)  & (0.09, 0.636, 2.888)  &  (0.10, 0.633, 2.875)  \\ 
 \bottomrule
  \end{tabular}
\end{table}
\begin{figure*}[!htb]
\centering
\begin{tabular}{cccc}
{\includegraphics[scale=0.16]
{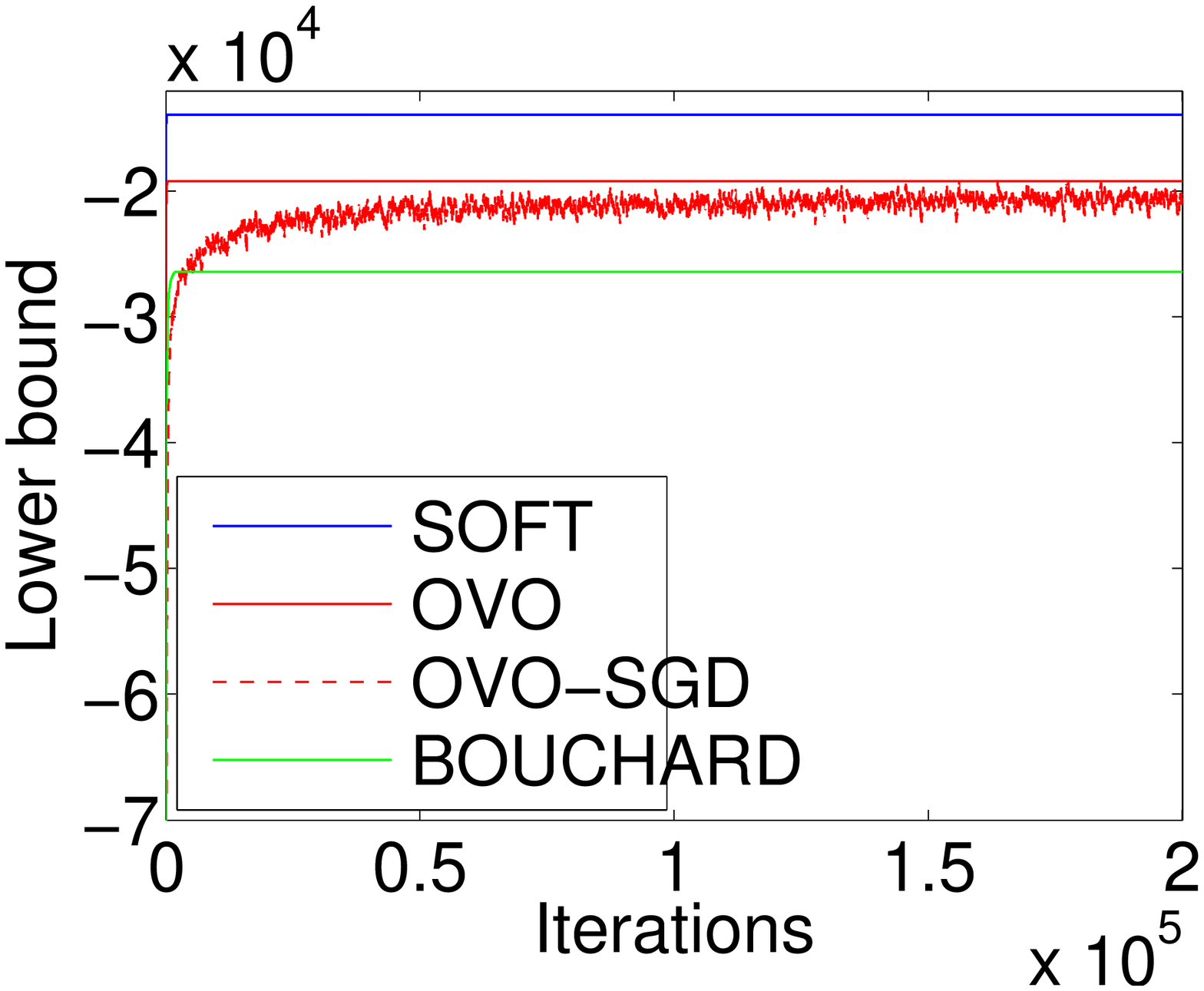}} &
{\includegraphics[scale=0.16]
{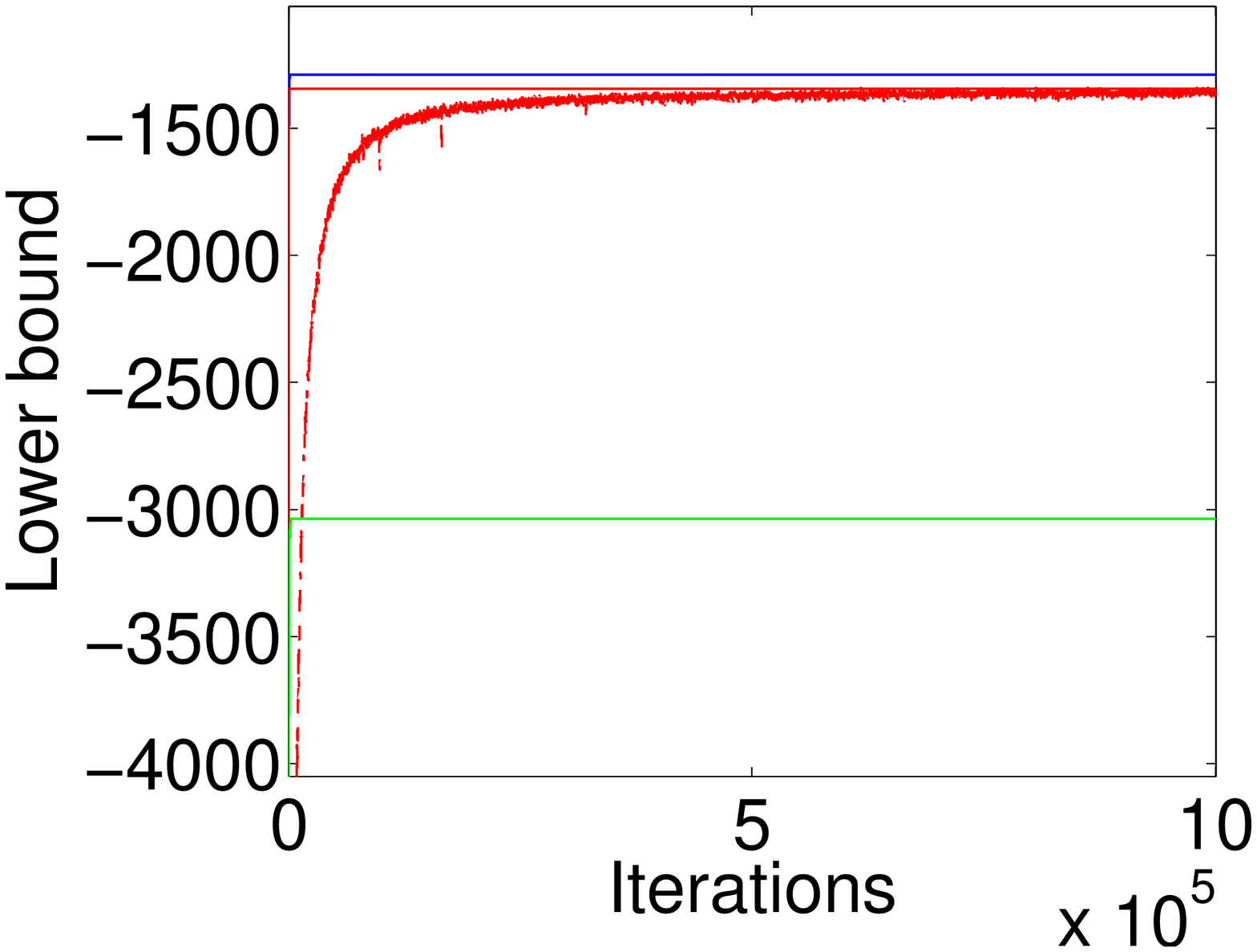}} &
{\includegraphics[scale=0.16]  
{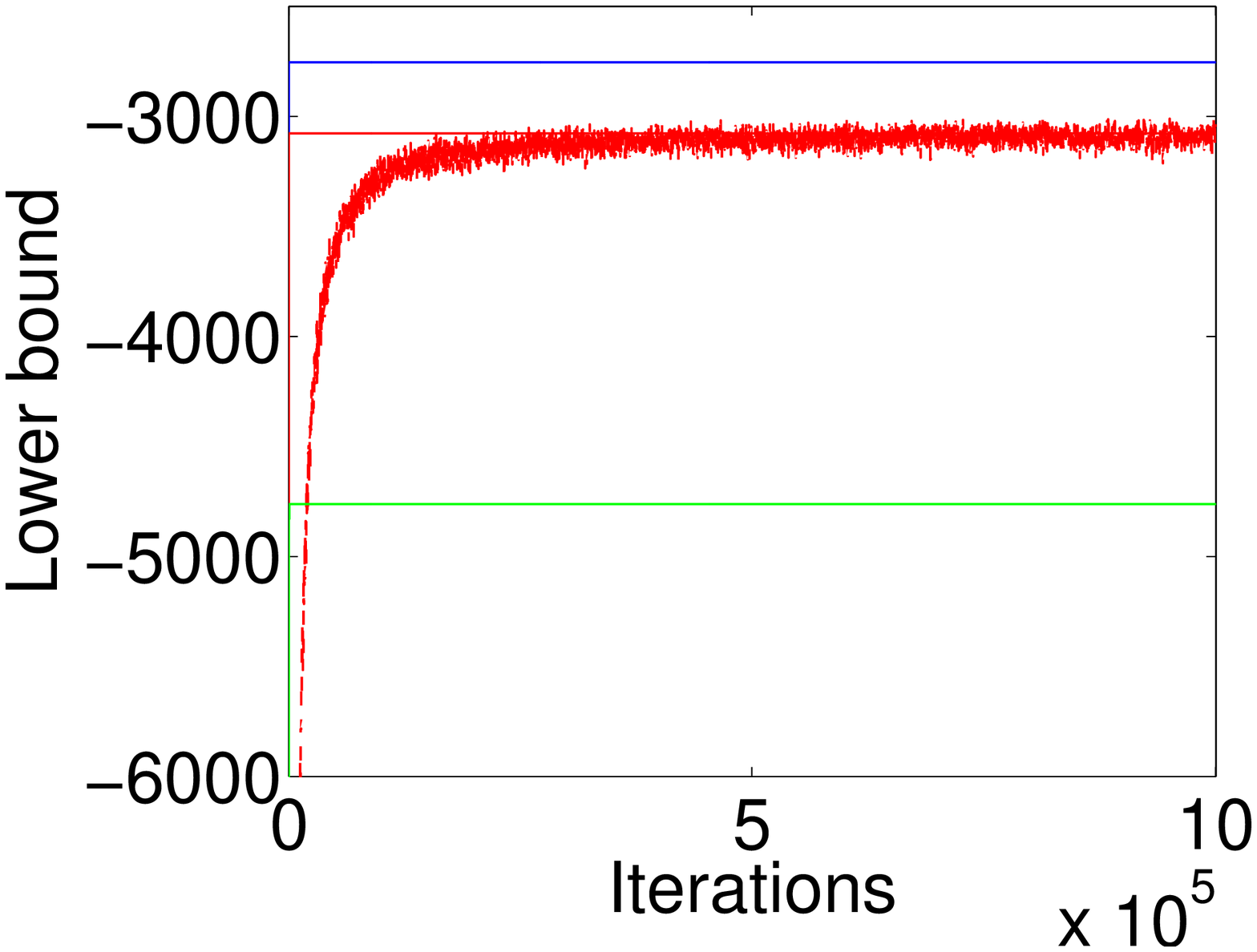}} &
{\includegraphics[scale=0.16]  
{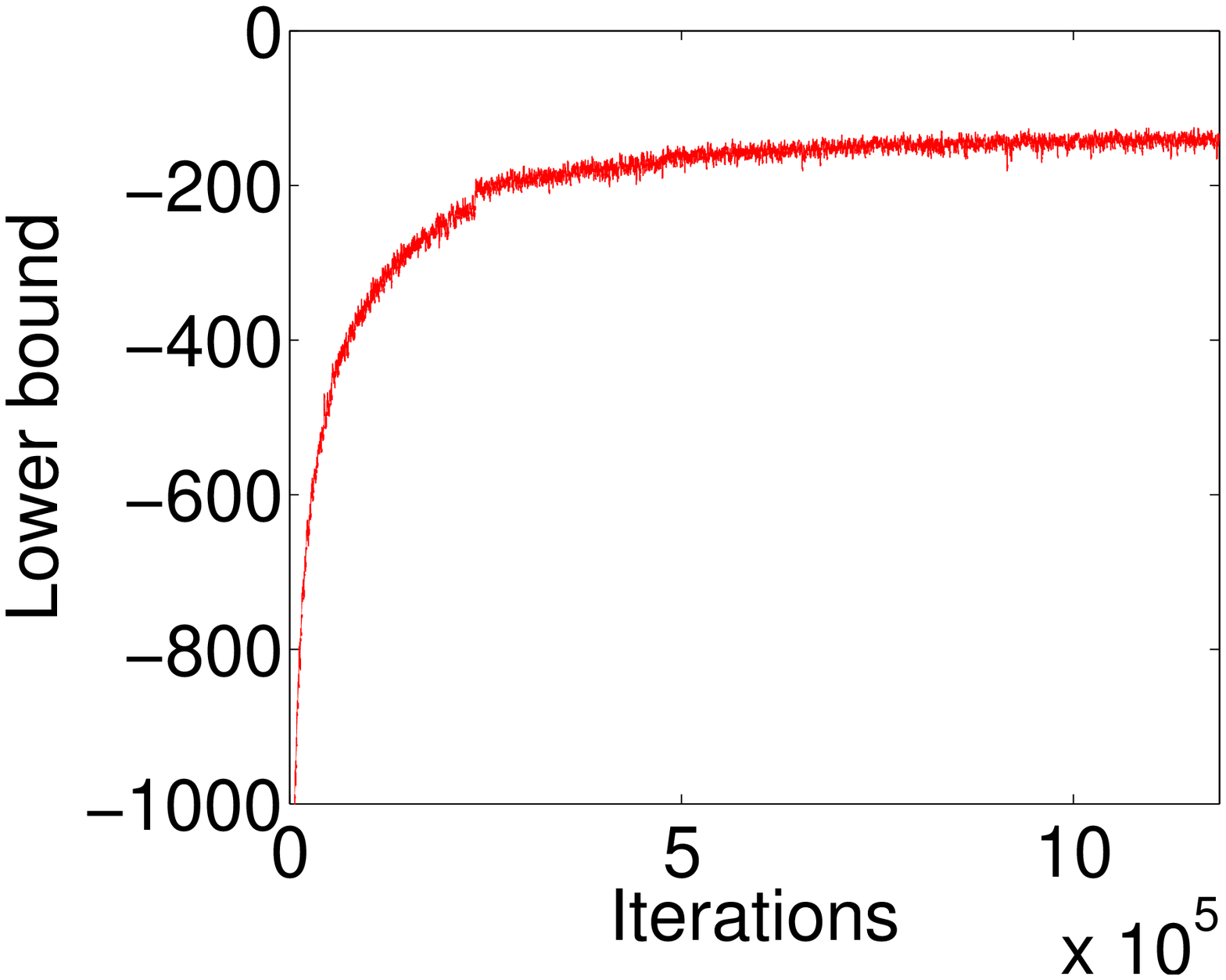}} \\
(a) & (b) & (c) & (d)
\end{tabular}
\vspace{-2mm}
\caption{(a) shows the evolution of the lower bound values for \textsc{mnist}, (b)  
for \textsc{20news} and (c) for \textsc{bibtex}. For more clear visualization the bounds of
the stochastic \textsc{ove-sgd} have been smoothed using a rolling window of $400$ previous values. (d) 
shows the evolution of the  \textsc{ove-sgd}  lower bound (scaled to correspond to a single data point) 
in the large scale \textsc{amazoncat-13k} dataset. Here, the plotted values 
have been also smoothed using a rolling window of size $4000$ and then thinned by a factor of $5$.
} 
\label{fig:smallScaleClass}
\end{figure*}

Further, notice that in this large dataset the number of parameters we need to estimate 
for the linear classification model is very large:
$K \times (D+1) = 2919 \times 203883$  parameters where the plus one accounts for the biases. 
All methods apart from \textsc{ove-sgd} are practically very slow in this massive dataset, 
and therefore we consider $\textsc{ove-sgd}$ which is scalable.    

We applied \textsc{ove-sgd} where at each stochastic gradient update we consider a single training 
instance (i.e.\ the minibatch size was one) 
and for that instance we randomly select five remaining classes. This leads to sparse parameter
updates, where the score function parameters of only six classes 
(the class of the current training instance 
plus the remaining five ones) are updated at each iteration. 
We used a very small learning rate having 
value $10^{-8}$ and we performed five epochs across the full dataset, that is 
we performed in total $5 \times 1186239$ stochastic gradient updates. After each epoch we halve the 
value of the learning rate before next epoch starts.  By taking into account also the sparsity of the 
input vectors each iteration is very fast and full training is completed in just $26$ minutes 
in a stand-alone PC. The evolution of the variational lower bound that indicates convergence is shown
in Figure \ref{fig:smallScaleClass}d. Finally, the classification error in test data was 
$53.11\%$ which is significantly better than random guessing or by a method that decides
 always the most populated class (where in \textsc{amazoncat-13k} the most populated class 
 occupies the $19 \%$ of the data so the error of that method is around $79\%$). 
  
\vspace{-2mm}

\section{Discussion}

\vspace{-2mm}

We have presented the one-vs-each lower bound on softmax probabilities and we have analyzed its theoretical 
properties. This bound is just the most extreme case of a full family
of hierarchically ordered bounds. We have explored the ability of the bound to perform parameter estimation 
through stochastic optimization in models having large number of categorical symbols, and we have demonstrated this ability 
to classification problems. 
  
There are several directions for future research. Firstly, it is worth investigating the 
usefulness of the bound in different applications from classification, such as for learning word embeddings
in natural language processing and for training recommendation systems. Another interesting 
direction is to consider the bound not for point estimation, as done in this paper, but for 
Bayesian estimation using variational inference. 

\vspace{-2mm}
\subsection*{Acknowledgments}
\vspace{-2mm}
We thank the reviewers for insightful comments. We would like also to thank Francisco J. R. Ruiz 
for useful discussions and David Blei for suggesting the name {\em one-vs-each} 
for the proposed method.  

\appendix 
\section{Proof of Proposition 3}

Here we re-state and prove {\bf Proposition 3}. 

{\bf Proposition 3.} {\em Assume that $K=2$ and we approximate the probabilities
$p(y=1)$ and $p(y=2)$ from $(2)$ 
 with the corresponding 
Bouchard's bounds given by $\frac{e^{f_1 - \alpha}}{(1 + e^{f_1 - \alpha}) 
(1 + e^{f_2 - \alpha})}$  and $\frac{e^{f_2 - \alpha}}{(1 + e^{f_1 - \alpha}) 
(1 + e^{f_2 - \alpha})}$. These bounds are used to approximate the maximum
likelihood solution for $(f_1,f_2)$ by maximizing the lower bound
\begin{equation}
\mathcal{F}(f_1,f_2,\alpha) = \log \frac{e^{N_1 (f_1 - \alpha)  + N_2 (f_2 -\alpha)}}{\left[(1 + e^{f_1 - \alpha}) 
(1 + e^{f_2 - \alpha})\right]^{N_1+N_2} }, 
\label{eq:Qf1f2alpha}
\end{equation}
obtained by replacing $p(y=1)$ and $p(y=2)$ 
in the exact log likelihood with Bouchard's bounds. 
Then, the global maximizer of $\mathcal{F}(f_1,f_2,\alpha)$ is such that 
\begin{equation}
\alpha = \frac{f_1 + f_2}{2}, \ \ 
f_k = 2 \log N_k  +  c, \ \ k=1,2.  
\label{eq:Bouchalphaf1f2_2}
\end{equation}
}
\begin{proof}
The lower bound is written as 
$$
N_1 (f_1 - \alpha)  + N_2 (f_2 - \alpha) 
- (N_1 + N_2) \left[ 
\log (1 + e^{f_1 -\alpha})  + \log (1 + e^{f_2 - \alpha})
 \right]. 
$$    
We will first maximize this quantity wrt $\alpha$. For that is suffices
to minimize the upper bound on the following log-sum-exp function
$$
\alpha + \log (1 + e^{f_1 -\alpha})  + \log (1 + e^{f_2 - \alpha}),
$$   
which is a convex function of $\alpha$. By taking the derivative wrt
$\alpha$ and setting to zero we obtain the stationary condition 
$$
\frac{e^{f_1 - \alpha}}{1 + e^{f_1 - \alpha}}  
  +  \frac{e^{f_2 - \alpha }}{1 + e^{f_2 - \alpha}} = 1.
$$ 
Clearly, the value of $\alpha$ that satisfies the condition 
is $\alpha = \frac{f_1 + f_2}{2}$. Now if we substitute 
this value back into the initial bound we have 
$$
N_1 \frac{f_1 - f_2}{2}  + N_2  \frac{f_2 - f_1}{2}
- (N_1 + N_2) \left[ 
\log (1 + e^{\frac{f_1 - f_2}{2}})  + \log (1 + e^{\frac{f_2 - f_1}{2}})
 \right] 
$$ 
which is concave wrt $f_1$ and $f_2$. Then, by taking derivatives wrt $f_1$ and $f_2$ 
we obtain the conditions  
$$
\frac{N_1 - N_2}{2}  =  \frac{(N_1 + N_2)}{2} \left[ 
\frac{ e^{\frac{f_1 - f_2}{2}}}{1 + e^{\frac{f_1 - f_2}{2}}} 
- 
\frac{ e^{\frac{f_2 - f_1}{2}}}{1 + e^{\frac{f_2 - f_1}{2}}} 
 \right]
$$
$$
\frac{N_2 - N_1}{2}  =  \frac{(N_1 + N_2)}{2} \left[ 
\frac{ e^{\frac{f_2 - f_1}{2}}}{1 + e^{\frac{f_2 - f_1}{2}}} 
- 
\frac{ e^{\frac{f_1 - f_2}{2}}}{1 + e^{\frac{f_1 - f_2}{2}}} 
 \right]
$$
Now we can observe that these conditions are satisfied by $f_1 = 2 \log N_1 + c$
and $f_2 = 2 \log N_2 + c$ which gives the global maximizer since 
$\mathcal{F}(f_1,f_2,\alpha)$ is concave.  
\end{proof}

\begin{small}
\bibliography{refs}

\begin{thebibliography}{}

\bibitem[Bengio and S{\'{e}}n{\'{e}}cal, 2003]{BengioSenecal-2003}
Bengio, Y. and S{\'{e}}n{\'{e}}cal, J.-S. (2003).
\newblock Quick training of probabilistic neural nets by importance sampling.
\newblock In {\em Proceedings of the conference on Artificial Intelligence and
  Statistics (AISTATS)}.

\bibitem[Bhatia et~al., 2015]{NIPS2015_5969}
Bhatia, K., Jain, H., Kar, P., Varma, M., and Jain, P. (2015).
\newblock Sparse local embeddings for extreme multi-label classification.
\newblock In Cortes, C., Lawrence, N.~D., Lee, D.~D., Sugiyama, M., and
  Garnett, R., editors, {\em Advances in Neural Information Processing Systems
  28}, pages 730--738. Curran Associates, Inc.

\bibitem[Bishop, 2006]{Bishop:2006}
Bishop, C.~M. (2006).
\newblock {\em Pattern Recognition and Machine Learning (Information Science
  and Statistics)}.
\newblock Springer-Verlag New York, Inc., Secaucus, NJ, USA.

\bibitem[Bohning, 1992]{Bohning92}
Bohning, D. (1992).
\newblock Multinomial logistic regression algorithm.
\newblock {\em Annals of the Inst. of Statistical Math}, 44:197--200.

\bibitem[Bouchard, 2007]{bouchard_efficient_2007}
Bouchard, G. (2007).
\newblock Efficient bounds for the softmax function and applications to
  approximate inference in hybrid models.
\newblock Technical report.

\bibitem[Bradley and Terry, 1952]{bradley1952rank}
Bradley, R.~A. and Terry, M.~E. (1952).
\newblock {Rank analysis of incomplete block designs: I. The method of paired
  comparisons}.
\newblock {\em Biometrika}, 39(3/4):324--345.

\bibitem[Devlin et~al., 2014]{devlin2014}
Devlin, J., Zbib, R., Huang, Z., Lamar, T., Schwartz, R., and Makhoul, J.
  (2014).
\newblock Fast and robust neural network joint models for statistical machine
  translation.
\newblock In {\em Proceedings of the 52nd Annual Meeting of the Association for
  Computational Linguistics (Volume 1: Long Papers)}, pages 1370--1380,
  Baltimore, Maryland. Association for Computational Linguistics.

\bibitem[Goodfellow et~al., 2016]{Goodfellow-et-al-2016-Book}
Goodfellow, I., Bengio, Y., and Courville, A. (2016).
\newblock Deep learning.
\newblock Book in preparation for MIT Press.

\bibitem[Gopal and Yang, 2013]{gopal13}
Gopal, S. and Yang, Y. (2013).
\newblock Distributed training of large-scale logistic models.
\newblock In Dasgupta, S. and Mcallester, D., editors, {\em Proceedings of the
  30th International Conference on Machine Learning (ICML-13)}, pages 289--297.
  JMLR Workshop and Conference Proceedings.

\bibitem[Huang et~al., 2006]{Huang:2006}
Huang, T.-K., Weng, R.~C., and Lin, C.-J. (2006).
\newblock Generalized {B}radley-{T}erry models and multi-class probability
  estimates.
\newblock {\em J. Mach. Learn. Res.}, 7:85--115.

\bibitem[Ji et~al., 2015]{BlackOut}
Ji, S., Vishwanathan, S. V.~N., Satish, N., Anderson, M.~J., and Dubey, P.
  (2015).
\newblock Blackout: Speeding up recurrent neural network language models with
  very large vocabularies.

\bibitem[Katakis et~al., 2008]{Katakis08multilabeltext}
Katakis, I., Tsoumakas, G., and Vlahavas, I. (2008).
\newblock Multilabel text classification for automated tag suggestion.
\newblock In {\em In: Proceedings of the ECML/PKDD-08 Workshop on Discovery
  Challenge}.

\bibitem[Khan et~al., 2012]{KhanMMM12}
Khan, M.~E., Mohamed, S., Marlin, B.~M., and Murphy, K.~P. (2012).
\newblock A stick-breaking likelihood for categorical data analysis with latent
  {G}aussian models.
\newblock In {\em Proceedings of the Fifteenth International Conference on
  Artificial Intelligence and Statistics, {AISTATS} 2012, La Palma, Canary
  Islands, April 21-23, 2012}, pages 610--618.

\bibitem[Mikolov et~al., 2013]{mikolov2013}
Mikolov, T., Sutskever, I., Chen, K., Corrado, G.~S., and Dean, J. (2013).
\newblock Distributed representations of words and phrases and their
  compositionality.
\newblock In Burges, C. J.~C., Bottou, L., Welling, M., Ghahramani, Z., and
  Weinberger, K.~Q., editors, {\em Advances in Neural Information Processing
  Systems 26}, pages 3111--3119. Curran Associates, Inc.

\bibitem[Mnih and Teh, 2012]{MnihTeh2012}
Mnih, A. and Teh, Y.~W. (2012).
\newblock A fast and simple algorithm for training neural probabilistic
  language models.
\newblock In {\em Proceedings of the 29th International Conference on Machine
  Learning}, pages 1751--1758.

\bibitem[Morin and Bengio, 2005]{morin2005hierarchical}
Morin, F. and Bengio, Y. (2005).
\newblock {Hierarchical probabilistic neural network language model}.
\newblock In {\em Proceedings of the Tenth International Workshop on Artificial
  Intelligence and Statistics}, pages 246--252. Citeseer.

\bibitem[Paquet et~al., 2012]{PaquetKoenigsteinWinther14}
Paquet, U., Koenigstein, N., and Winther, O. (2012).
\newblock Scalable {B}ayesian modelling of paired symbols.
\newblock {\em CoRR}, abs/1409.2824.

\bibitem[Pennington et~al., 2014]{pennington-etal-2014}
Pennington, J., Socher, R., and Manning, C. (2014).
\newblock {Glove: Global Vectors for Word Representation}.
\newblock In {\em Proceedings of the 2014 Conference on Empirical Methods in
  Natural Language Processing (EMNLP)}, pages 1532--1543, Doha, Qatar.
  Association for Computational Linguistics.

\bibitem[Vijayanarasimhan et~al., 2014]{VijayanarasimhanSMY14}
Vijayanarasimhan, S., Shlens, J., Monga, R., and Yagnik, J. (2014).
\newblock Deep networks with large output spaces.
\newblock {\em CoRR}, abs/1412.7479.

\end{thebibliography}
\bibliographystyle{apalike}
\end{small}

\end{document}